\documentclass[11pt]{article}

\usepackage{acl}
\usepackage{amsthm}
\definecolor{lightgreen}{RGB}{230, 255, 230}
\definecolor{lightyellow}{RGB}{255, 255, 230}
\definecolor{lightred}{RGB}{255, 230, 230}
\definecolor{lightblue}{RGB}{230, 240, 255}
\usepackage[table]{xcolor}
\usepackage{booktabs, siunitx, xcolor}
\usepackage{makecell}
\usepackage{geometry} 
\usepackage{microtype} 
\usepackage{lipsum}    
\usepackage{enumitem}  
\usepackage{url}       
\usepackage{changes}   
\usepackage{xspace}    
\usepackage{float}

\usepackage{amsmath}    
\usepackage{amssymb}    
\usepackage{mathtools}  
\usepackage{calc}       

\usepackage{booktabs}       
\usepackage{tabularx}       
\usepackage{graphicx}       
\usepackage{caption}        
\usepackage{subcaption}     
\usepackage{pgfplots}       
\usepackage{pgfplotstable}  
\pgfplotsset{compat=1.17}   

\usepackage{tikz} 
\usepackage{pgfplots} 
\pgfplotsset{compat=1.18} 

\usepackage{algorithm}     
\usepackage{algorithmicx}  
\usepackage{algpseudocode} 
\usepackage{url}

\usepackage{amsmath, amssymb}
\usepackage{xcolor}
\usepackage{graphicx}

\usepackage{tikz}
\usetikzlibrary{arrows.meta,calc,patterns,positioning}
\usepackage{lmodern}
\usepackage{pgfplots}
\pgfplotsset{compat=1.18} 
\usepgfplotslibrary{polar} 
\usepackage{siunitx}       
\usepackage{booktabs}      
\usepackage[table]{xcolor} 

\newtheorem{lemma}{Lemma}[section] 

\usepackage{amsmath} 

\usepackage{titlesec}
\titleformat{\section}{\bfseries\large}{\thesection}{1em}{}
\titlespacing*{\section}{0pt}{3pt}{2pt}  
\titlespacing*{\subsection}{0pt}{2pt}{1pt}

\usepackage{amsmath,amssymb}
\usepackage{booktabs}
\usepackage[utf8]{inputenc}
\usepackage[T1]{fontenc}
\usepackage{enumitem} 

\theoremstyle{plain}
\newtheorem{theorem}{Theorem}[section]

\newcommand{\E}{\mathbb{E}}
\newcommand{\grad}{\nabla}
\newcommand{\Cov}{\text{Cov}}
\newcommand{\Hent}{H}
\newcommand{\pitheta}{\pi_\theta}
\newcommand{\thetanew}{\theta_{\text{new}}}
\newcommand{\thetaold}{\theta}
\newcommand{\aaction}{a}
\newcommand{\sstate}{s}
\newcommand{\Aadv}{A}
\newcommand{\phiscore}{\phi}
\newcommand{\Ffisher}{\mathcal{F}}

\newcommand{\method}{\texttt{TEPO}\xspace}

\title{Token-Level Policy Optimization: Linking Group-Level Rewards to Token-Level Aggregation via Sequence-Level Likelihood}
\author{Xingyu Lin$^{1,2,3}$,
  Yilin Wen$^{3}$,
  Du Su$^{5}$,
  Jinchang Hou$^{3}$,\\
  \textbf{En Wang}$^{1,2\ \ast}$,
  \textbf{Wenbin Liu}$^{1,2\ \ast}$,
  \textbf{Chenfu Bao}$^{3,4\ \ast}$,
  \textbf{Zhonghou Lv}$^{3}$
  \thanks{These authors contributed equally to this work.\\ Corresponding authors: En Wang (enwang@jlu.edu.cn),  Wenbin Liu (liuwenbin@jlu.edu.cn), Chenfu Bao (chenfubao@baidu.com,  bcf25@mails.tsinghua.edu.cn), Zhonghou Lv (zhonghoulv@baidu.com)}
  \\
  $^1$College of Computer Science and Technology, Jilin University \\
  $^2$Key Laboratory of Symbolic Computation and Knowledge Engineering of MOE, Jilin University \\
  $^3$Baidu Inc. $^4$Tsinghua University.\\
  $^5$State Key Laboratory of AI Safety, Institute of Computing Technology, Chinese Academy of Sciences \\ \ \\ \ \\
}

\begin{document}

\maketitle
\begin{abstract}
Group Relative Policy Optimization (GRPO) has significantly advanced the reasoning ability of large language models (LLMs), particularly in their mathematical reasoning performance. However, GRPO and related entropy regularization methods still struggle with token-level sparse-rewards, which is an inherent challenge in chain-of-thought (CoT) reasoning. These approaches often rely on undifferentiated token-level entropy regularization, which easily leads to entropy collapse or model degradation under sparse token rewards. In this work, we propose \method, a novel token-level framework that (1) leverages sequence-level likelihood to link group-level rewards with individual tokens via token-level aggregation, and (2) introduces a token-level KL-Divergence mask constraint that targets tokens with positive advantages and decreasing entropy to mitigate abrupt policy updates. Experiments demonstrate that \method not only achieves state-of-the-art performance on mathematical reasoning benchmarks but also markedly enhances training stability, reducing convergence time by 50\% compared with GRPO/DAPO.
\end{abstract}
\section{Introduction}
GRPO\cite{shao2024deepseekmath}  has significantly advanced the reasoning ability of large language models (LLMs), particularly in mathematical reasoning. However, in CoT reasoning, learning is fundamentally challenged by sparse token-level rewards, under which GRPO and related entropy-regularized methods often struggle. Specifically, these approaches rely on undifferentiated token-level entropy regularization, which can lead to entropy collapse or policy degradation when rewards are sparse\cite{yu2025dapo}.
\begin{figure}[!t]
    \includegraphics[width=\linewidth]{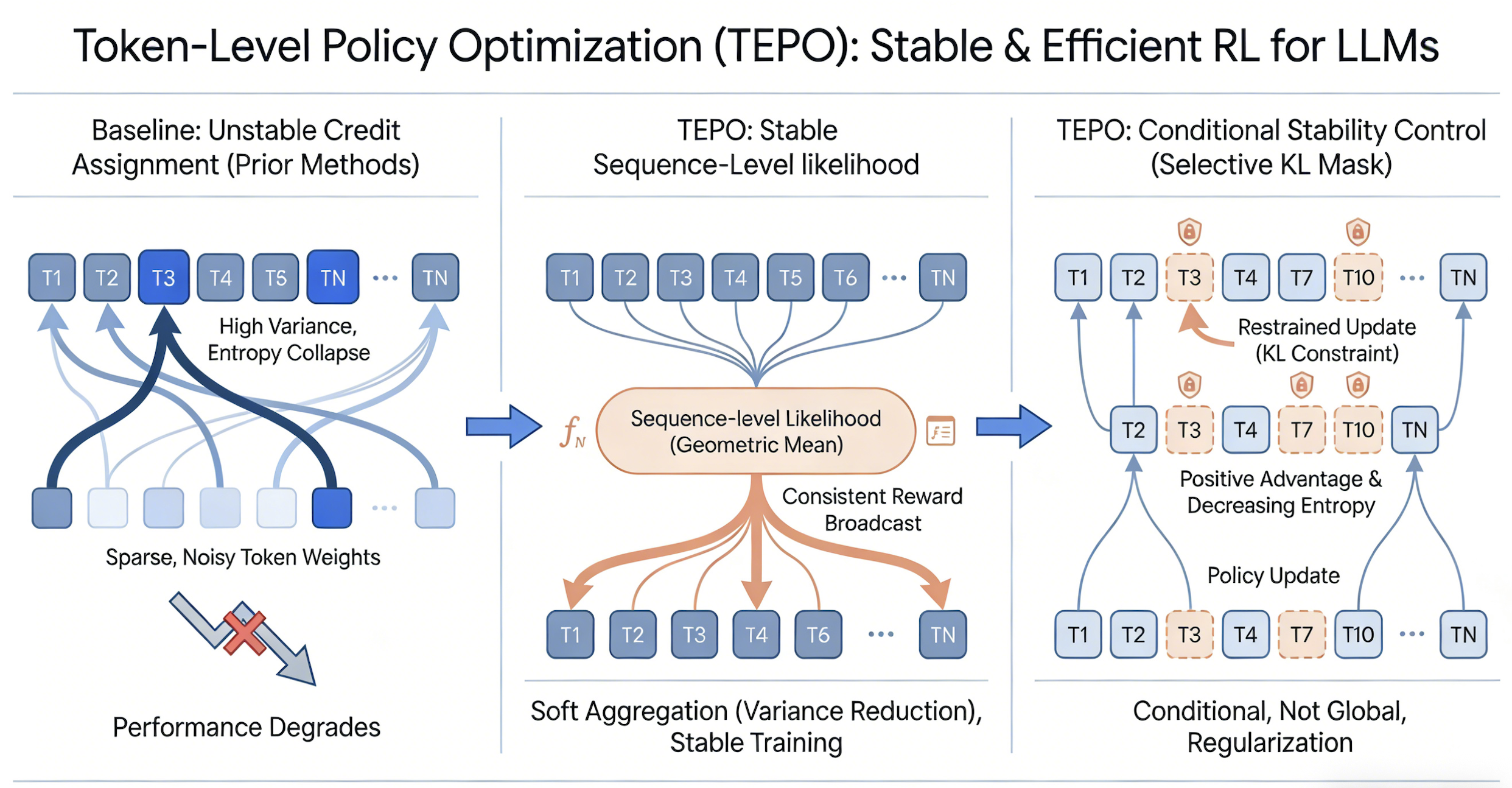}
    \caption{Overview of the \method Framework
    \method (1) replaces baselines’ noisy, sparse token-level credit assignment with sequence-level likelihood, using soft aggregation to broadcast group rewards to tokens and stabilize training. (2) A selective KL mask curbs abrupt updates exclusively for tokens with positive advantage and decreasing entropy, balancing entropy reduction and stability.}
    \label{fig:Overview_of_tepo_Framework}
\end{figure}
For entropy regularization, methods either minimize entropy to ensure credible outputs \cite{agarwal2025unreasonable} or maximize it to enhance exploration; however, such undifferentiated entropy adjustments yield only marginal improvements. For KL-Divergence, this regularization strategy is fragile and degrades final performance without extensive parameter tuning, manifesting as entropy collapse or model collapse \cite{gspo2025}. 

\textbf{Core Insight:} These challenges are rooted in the inherent sparse-reward of CoT reasoning, which creates instability that GRPO’s critic-free framework struggles to mitigate. The absence of a critic in GRPO \cite{shao2024deepseekmath}, combined with the token-level sparse-reward nature of long-chain reasoning tasks, exacerbates its susceptibility to high-variance gradient estimates. Policies exploring novel CoT structures often diverge substantially from their initial distribution \cite{cheng2025reasoning}, leading to cumulative noise across extended reasoning sequences \cite{zheng2025group}. Furthermore, sole reliance on undifferentiated entropy regularization or KL-Divergence exacerbates this issue by triggering model collapse in sparse-reward CoT settings \cite{chu2025gpg, gao2025one, zheng2025group}, further undermining training stability.

To address this token-level sparse-reward, we propose \method, a  token-level framework designed to align group-level rewards with token-level credit assignment. \method (1) leverages sequence-level likelihood to bridge group-level rewards with individual tokens via token-level aggregation, and (2) introduces a token-level KL-Divergence mask (applied to tokens with positive advantage and decreasing entropy) to mitigate abrupt policy updates. Notably, \cite{cui2025entropy} establishes a well-documented relationship between model performance ($R$) and policy entropy ($\mathcal{H}$): $R = -a \cdot \exp(\mathcal{H}) + b$ , revealing that performance improvements are fundamentally achieved through systematic entropy reduction. By resolving state distribution shift  with sequence-level likelihood(Section \ref{sec:Sequence-level Likelihood}), \method enables trading larger $\Delta\mathcal{H}$ for better downstream performance. Additionally, the token-level KL-Divergence mask mitigates rapid entropy decay. Our key contributions are summarized as follows:
\begin{itemize}[leftmargin=0.8em, itemsep=0pt, parsep=0pt, topsep=0pt, labelsep=0.3em]
\item \textbf{Token-Level Policy Optimization}: We propose a highly efficient and adaptive token-level optimization strategy tailored to critic-free paradigms. Notably, our method achieves peak performance in only 72 optimization steps, while GRPO/DAPO require 132 steps to reach comparable performance, reducing convergence time by nearly 50\%.
\item \textbf{Analysis of KL-Divergence and Entropy Regularization in GRPO}: We provide both theoretical and empirical evidence demonstrating that undifferentiated KL-Divergence and entropy regularization struggle in sparse-reward settings.
\item \textbf{Comprehensive Experimental Validation}: Our method achieves an approximately 2\% improvement in average accuracy over the baseline GRPO. We further conduct ablation studies to verify that GRPO/DAPO and related entropy regularization methods exhibit performance bottlenecks under token-level sparse-reward conditions, and the results confirm the effectiveness of \method. 
\end{itemize}
\section{Preliminaries in LLMs}
We list all core notations in Table \ref{tab:notations}.
\subsection{Entropy Gradient Derivation}
In LLMs, the state $s$ corresponds to the prompt context, and an action $a$ refers to a token from the vocabulary \(\mathcal{A}\). For a state s and action a, the policy is defined as follows:
\[
\pi_\theta(a \mid s) = \frac{\exp\!\left(\phi_\theta(s,a)\right)}{\sum_{a' \in \mathcal{A}} \exp\!\left(\phi_\theta(s,a')\right)},
\]
which is a softmax function.  Here, $\phi_\theta(s,a)\in\mathbb{R}$ is the token logit. The entropy of policy distribution $\pi_\theta(\cdot \mid s)$ measures its uncertainty: \[ \mathcal{H}\left(\pi_\theta(\cdot \mid s)\right) = -\sum_{a} \pi_\theta(a \mid s) \log \pi_\theta(a \mid s). \] Applying the chain rule yields $ \frac{\partial\mathcal{H}}{\partial \phi_\theta(a_i \mid s)}$: 
\begin{equation} \label{eq:entropy_gradient}  
\pi_\theta(a_i \mid s) \left( \log \pi_\theta(a_i \mid s) + \mathcal{H}\left(\pi_\theta(\cdot \mid s)\right) \right). 
\end{equation}

The partial derivatives are given by: 
\[\small \frac{\partial \pi_\theta(a_i \mid s)}{\partial \phi_\theta(s, a_i)} = \begin{cases} \pi_\theta(a \mid s) \left(1 - \pi_\theta(a \mid s)\right), & a = a_i \\ -\pi_\theta(a \mid s) \pi_\theta(a_i \mid s), & a \neq a_i \end{cases}, \]
\[ \frac{\partial \log \pi_\theta(a_i \mid s)}{\partial \phi_\theta(s, a_i)} = \begin{cases} 1 - \pi_\theta(a_i \mid s), & a = a_i \\ -\pi_\theta(a_i \mid s), & a \neq a_i \end{cases}. \]
\subsection{Policy Gradient Derivation}
Proximal Policy Optimization (PPO) \cite{schulman2017proximal} and RLVR \cite{lambert2024tulu}  aim to maximize a rule-based reward $A_t = \frac{r\left( \boldsymbol{y} \right) - \operatorname{mean}\left( r\left( \boldsymbol{y}^{1:G} \right) \right)}{\operatorname{std}\left( r\left( \boldsymbol{y}^{1:G} \right) \right)}$ \cite{williams1992simple} :
\begin{equation}
\max_{\theta} J(\theta) := \mathbb{E}_{\boldsymbol{x} \sim \mathcal{D}, \boldsymbol{y} \sim \pi_\theta(\boldsymbol{x})} \left[ A(\boldsymbol{y}) \right],
\end{equation}
where $\boldsymbol{x} \sim \mathcal{D}$ is the input prompt, and $\boldsymbol{y} = \{y_1, \dots, y_T\}$ is the generated sequence of length $T$.  Plugging in the partial derivatives into the policy gradient framework, we obtain:
\begin{equation} 
\label{eq:policy_gradient}
\begin{split}
&\frac{\partial J}{\partial \phi_\theta(s,       a_i)}=\mathbb{E}_{\pi_\theta} \frac{\partial \log \pi_\theta(a_i \mid s)}{\partial \phi_\theta(s, a_i)} A(s, a) \\
&= \pi_\theta(a_i \mid s) \left( A(s, a_i) - \mathbb{E}_{\pi_\theta} \left[ A(s, a) \right] \right),
\end{split}
\end{equation}
where $\mathbb{E}_{\pi_\theta} \left[ A(s, a) \right] = \sum_{a} \pi_\theta(a \mid s) A(s, a)$.
\section{Methodology}\label{sec:method}
\subsection{Framework of \method}
In Fig.\ref{fig:Overview_of_tepo_Framework}, \method leverages sequence-level likelihood and a token-level selective KL-Divergence mask to bridge group-level rewards with individual tokens. Our method's loss function is formalized as Equation \ref{eq:loss_function_TEPO}, where the blue component denotes token-level aggregation, the green component represents sequence-level likelihood, and the red component ($M_{i,t}$) corresponds to the KL-Divergence mask (applied to the $t$-th token in the $i$-th response).
\begin{figure*}[t!]
\centering
\setlength{\abovedisplayskip}{0pt}  
\setlength{\belowdisplayskip}{0pt}  
\begin{equation}
\small  
\begin{aligned}\label{eq:loss_function_TEPO}
J(\theta) = \textcolor{blue}{\frac{1}{\sum_{i=1}^G \|o_i\|} \sum_{i=1}^G \sum_{t=1}^{\|o_i\|}} \Bigg[ &\min\bigg( \textcolor{green}{w_i(\theta)}, \operatorname{clip}\big( \textcolor{green}{w_i(\theta)}, 1-\epsilon, 1+\epsilon \big) \bigg) \cdot A_{i,t} - \beta \cdot \text{KL}\big( \pi_{\theta} \parallel \pi_{\theta_{\text{old}}} \big) \cdot \textcolor{red}{M} \Bigg], \\
M &=
\begin{cases}
    1, & \text{if } A_{i,t}>0 \land \Delta\mathcal{H}_{i,t}<0, \\
    0, & \text{otherwise}.
\end{cases}
\end{aligned}
\end{equation}
\end{figure*}
\subsection{Limitations of Critic-Free GRPO}
\subsubsection{The Fragility of the KL-Divergence}
\begin{lemma}
For a softmax policy $\pi_\theta(a \mid s)$ at iteration $k+1$, the update rule is:
\[
\pi_{k+1}(a \mid s) \propto \pi_k(a \mid s) \exp\left(\beta^{-1} A(s, a)\right),
\]
where $\beta$ balances stability and performance. 
\end{lemma}
\begin{proof}
At iteration $k$, we obtain the next policy $\pi_{k+1}$, subject to $\sum_{a} \pi_{k+1} (a \mid s) = 1$ :
\[
\max_p \  \mathbb{E}_{a \sim p} \left[ A(s, a) \right] - \beta \cdot \text{KL} \left( \pi_{k+1} \| \pi_k(\cdot \mid s) \right)
\]
Formulating the Lagrangian optimization problem yields an iterative policy update rule:
\begin{equation} \small \label{eq:policy_update}
\pi_{k+1}(a \mid s) = \frac{\pi_k(a \mid s) \exp\left(\beta^{-1} A(s, a)\right)}
{\mathbb{E}_{a' \sim \pi_k(\cdot \mid s)} \left[ \exp\left(\beta^{-1} A(s, a')\right) \right]}.
\end{equation}
Equivalently, the update can be written as:
\[
\pi_{k+1}(a \mid s) \propto \pi_k(a \mid s) \exp\left(\beta^{-1} A(s, a)\right).
\]
\textbf{Remark:}  KL-Divergence constraint preserves stability while hurts performance. Supporting evidence is shown in Panel A of Table \ref{tab:ablation_combined}.
\end{proof}
\subsubsection{Inner Product Between Entropy Gradient and Policy Gradient}
\label{sec:Misalignment Between Entropy Gradient and Policy Gradient}
\begin{lemma}
For a softmax policy $\pi_\theta(a \mid s) \propto \exp(\phi_\theta(s,a))$, the alignment between entropy gradient and policy gradient exhibits as:
\begin{itemize}[leftmargin=0.8em, itemsep=1pt, parsep=0pt, topsep=0pt, labelsep=0.3em]
    \item For suboptimal actions ($A(s,a) < 0$): $\langle \nabla_{\phi_\theta} \mathcal{H}, \nabla_{\phi_\theta} J \rangle > 0$
    \item For optimal actions ($A(s,a) > 0$): $\langle \nabla_{\phi_\theta} \mathcal{H}, \nabla_{\phi_\theta} J \rangle < 0$
\end{itemize}
\end{lemma}
\begin{proof}
We analyze the inner product between the entropy gradient and the policy gradient:
\[
\langle \nabla_{\phi_\theta} \mathcal{H}, \nabla_{\phi_\theta} J \rangle = \sum_{a} \frac{\partial \mathcal{H}}{\partial \phi_\theta(s, a)} \cdot \frac{\partial J}{\partial \phi_\theta(s, a)}.
\]
Substituting the entropy gradient from Eq.~\ref{eq:entropy_gradient} and the policy gradient from Eq.~\ref{eq:policy_gradient} yields:
\begin{equation*}
\small
\begin{split}
&\langle \nabla_{\phi_\theta} \mathcal{H}, \nabla_{\phi_\theta} J \rangle = \sum_{a} \frac{\partial \mathcal{H}}{\partial \phi_\theta(s, a)} \cdot \frac{\partial J}{\partial \phi_\theta(s, a)}. \\
&= \sum_{a_i} \pi_\theta(a_i \vert s)^2 \left( \log \pi_\theta(a_i \vert s) + \mathcal{H}(\pi_\theta(\cdot \vert s)) \right) A(s, a_i).
\end{split}
\end{equation*}

\textbf{Case 1: Suboptimal Actions } 
\begin{itemize}[leftmargin=0.8em, itemsep=1pt, parsep=0pt, topsep=0pt, labelsep=0.3em]
    \item For actions with small probability ($\pi_\theta(a \mid s) \to 0^{+}$), we have $\log \pi_\theta(a \mid s) \to -\infty$ and $\mathcal{H}(\pi_\theta(\cdot \mid s)) \to 0^{+}$, making $\log \pi_\theta(a \mid s) + \mathcal{H}(\pi_\theta(\cdot \mid s)) < 0$.
    \item With $\pi_\theta(a \mid s)^2 > 0$, $A(s,a) < 0$, therefore: $\langle \nabla_{\phi_\theta} \mathcal{H}, \nabla_{\phi_\theta} J \rangle > 0$.
\end{itemize}

\textbf{Case 2: Optimal Actions} 
\begin{itemize}[leftmargin=0.8em, itemsep=1pt, parsep=1pt, topsep=0pt, labelsep=0.3em]
    \item For actions with large probability ($\pi_\theta(a \mid s) \to 1^{-}$), we have $\log \pi_\theta(a \mid s) \to 0^-$, making $\log \pi_\theta(a \mid s) + \mathcal{H}(\pi_\theta(\cdot \mid s))=(1-\pi_\theta(a \mid s))\log \pi_\theta(a \mid s)<0$
    \item With $\pi_\theta(a \mid s)^2 > 0$, $A(s,a) > 0$, therefore: $\langle \nabla_{\phi_\theta} \mathcal{H}, \nabla_{\phi_\theta} J \rangle < 0$.
\end{itemize}
\end{proof}
\subsubsection{Undifferentiated Entropy Regularization unfits}
\begin{theorem}
$\Delta \mathcal{H}$ (entropy change) characterizes the exploration tendency as follows:
\begin{itemize}[leftmargin=0.8em, itemsep=1pt, parsep=1pt, topsep=0pt, labelsep=0.3em]
\item For $A(s, a) < 0$: $(\langle \nabla \mathcal{H}, \nabla J \rangle > 0) \rightarrow \Delta \mathcal{H} > 0$, \textbf{promoting} exploration.
\item For $A(s, a) > 0$: $(\langle \nabla \mathcal{H}, \nabla J \rangle < 0) \rightarrow \Delta \mathcal{H} < 0$, \textbf{suppressing} exploration.
\end{itemize}
\end{theorem}
\begin{proof}
Policy gradient update with $\alpha > 0$ is:
\[
\phi_\theta(s, a) \leftarrow \phi_\theta(s, a) + \alpha \cdot \frac{\partial J}{\partial \phi_\theta(s, a)}.
\]
$\Delta \mathcal{H}$ is approximated via a first-order Taylor expansion as:
\[
\Delta \mathcal{H} \approx\frac{\partial \mathcal{H}}{\partial \phi_\theta(s, a)} \cdot \Delta \phi_\theta(s, a) = \alpha \cdot \langle \nabla_{\phi_\theta} \mathcal{H}, \nabla_{\phi_\theta} J \rangle,
\]
where $\Delta \phi_\theta(s, a) = \alpha \cdot \frac{\partial J}{\partial \phi_\theta(s, a)}$. Therefore:
\begin{itemize}[leftmargin=0.8em, itemsep=1pt, parsep=1pt, topsep=0pt, labelsep=0.3em]
    \item When $A(s,a) < 0$ , $(\langle \nabla \mathcal{H}, \nabla J \rangle > 0)\rightarrow\Delta \mathcal{H} > 0$, \textbf{promoting} exploration.
    \item When $A(s,a) > 0$ , $(\langle \nabla \mathcal{H}, \nabla J \rangle < 0)\rightarrow\Delta \mathcal{H} < 0$, \textbf{suppressing} exploration.
\end{itemize}
\textbf{Remark:} GRPO grants a mechanism that suppresses high-advantage exploration and redirects it to low-advantage regions in critic-free GRPO, which counteracts with undifferentiated entropy regularization. Evidence that undifferentiated entropy regularization is unsuitable is provided in Panel B of Table \ref{tab:ablation_combined}.
\end{proof}
\subsubsection{Why Token-Level Importance Sampling unfits} \label{Why Token-Level Importance Sampling unfits}
In GRPO \cite{shao2024deepseekmath}, we decompose the policy entropy change $\mathcal{H}(\pi_{k+1}) - \mathcal{H}(\pi_k)$, under the state distributions $d^{\pi_k}$ and $d^{\pi_{k+1}}$ :
\begin{equation*}
\label{eq:entropy_decomposition}
\begin{split}
&\underbrace{\mathbb{E}_{s \sim d^{\pi_{k+1}}} \mathcal{H}(\pi_{k+1}(\cdot \mid s)) - \mathbb{E}_{s \sim d^{\pi_k}} \mathcal{H}(\pi_{k+1}(\cdot \mid s))}_{\text{\textcolor{purple}{State Distribution Shift: $\Delta \mathcal{H}$ with Importance Sampling}}} \\
&+ \underbrace{\mathbb{E}_{s \sim d^{\pi_k}} \mathcal{H}(\pi_{k+1}(\cdot \mid s)) - \mathbb{E}_{s \sim d^{\pi_k}} \mathcal{H}(\pi_k(\cdot \mid s))}_{\text{$\Delta \mathcal{H}$ During Sampling}}
\end{split}
\end{equation*}
The discrepancy between its sampling strategy and model update strategy gives rise to \textcolor{purple}{State Distribution Shift: $\Delta \mathcal{H}$ with Importance Sampling}. Consequently, critic-free GRPO struggles to perform token-level importance sampling (IS) $\frac{\pi_\theta(y_{i,t} \mid x, y_{i,<t})}{\pi_{\theta_{\text{old}}}(y_{i,t} \mid x, y_{i,<t})}$, as this token-level metric fails to capture the global state distribution shift ($d^{\pi_{k+1}} \neq d^{\pi_k}$). We further identify that critic-free GRPO inherently suffers from sparse token-level rewards, a finding validated by ablation studies in Panel C of Table \ref{tab:ablation_combined} and formal mathematical derivations in Section \ref{sec:Why Sentence Likelihood Derivation Fits GRPO}.
\begin{figure*}[tb]
\centering
\setlength{\parindent}{0pt}
\begin{subfigure}{0.45\textwidth}
\centering
\includegraphics[width=\linewidth]{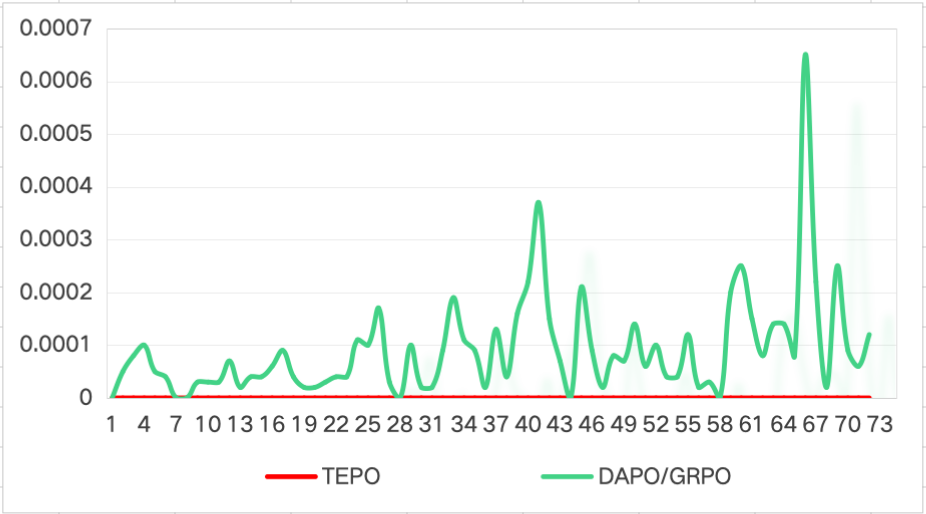}
\subcaption{Clip Ratio Over Training Steps}
\label{fig:Clip_Ratio_Comparison}
\end{subfigure}
\hspace{0.05\textwidth}
\begin{subfigure}{0.45\textwidth}
\centering
\includegraphics[width=\linewidth]{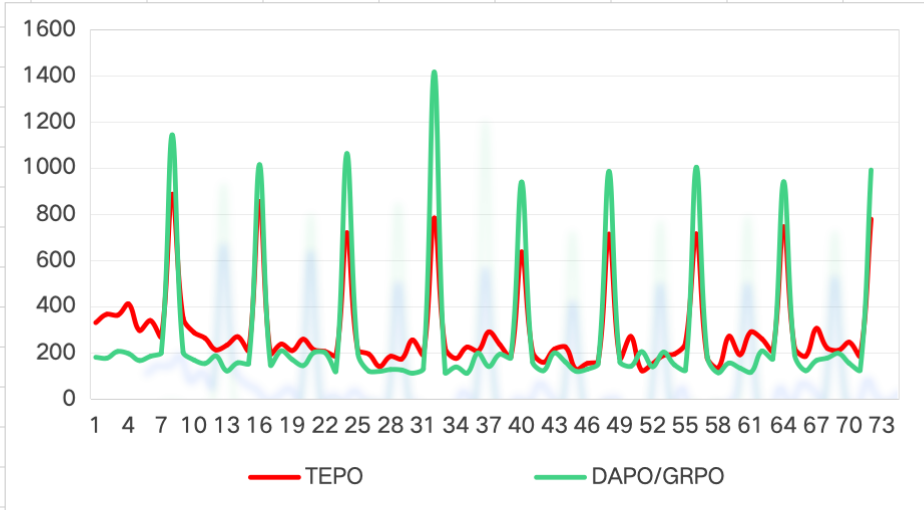}
\subcaption{Reasoning Time Evolution Over Training Steps}
\label{fig:Time_per_Step}
\end{subfigure}
\caption{Lower Gradient Bias and Faster Reasoning Efficiency with Markov Likelihood: The left panel shows that our method with a lower clip ratio effectively mitigates gradient bias, while the right panel indicates that it reduces the generation of redundant reasoning steps. Specifically, the average reasoning time of TEPO is 338 \textbf{seconds per step}, which is lower than that of DAPO/GRPO (357).}
\label{fig:dynamic_per_Step}
\end{figure*}
\begin{figure*}[tb]
\centering
\setlength{\parindent}{0pt}
\begin{subfigure}{0.45\textwidth}
\centering
\includegraphics[width=\linewidth]{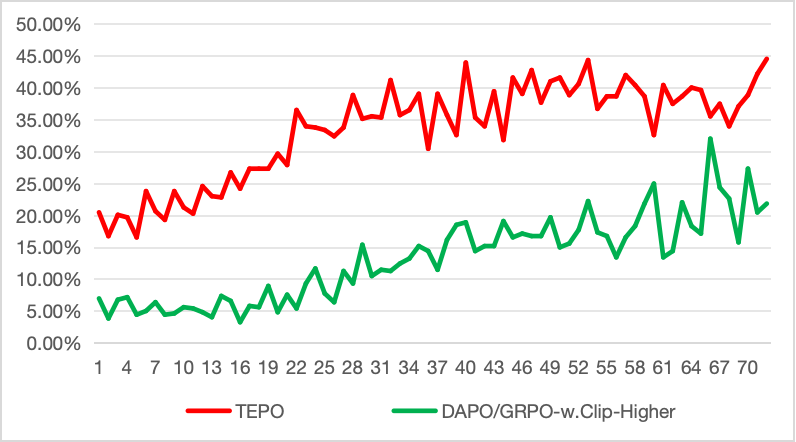}
\subcaption{Reward Progression Over Steps}
\label{fig:tepo_grpo_reward}
\end{subfigure}
\hspace{0.05\textwidth}
\begin{subfigure}{0.45\textwidth}
\centering
\includegraphics[width=\linewidth]{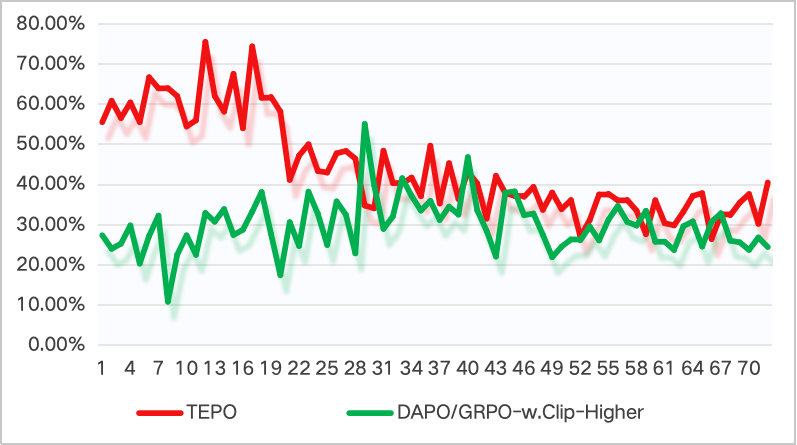}
\subcaption{Gradient Norm Over Steps}
\label{fig:tepo_grpo_grad}
\end{subfigure}
\caption{Markov Likelihood enhanced performance (We transform the raw data into percentage-based values.):  The left panel shows that our method achieves steadier and higher rewards across training steps, demonstrating more efficient learning dynamics. The right panel indicates that our method exhibits consistently higher gradient norms, reflecting more active and effective parameter reasoning.}
\label{fig:tepo_vs_grpo}
\end{figure*}
\subsection{Components in \method}
This section explains the rationale behind each component of our method.
\subsubsection{Sequence-level Likelihood} \label{sec:Sequence-level Likelihood}
Leveraging the Markov factorization (where each token’s probability depends on prior tokens \cite{chung1967markov}), we define the sequence-level weight $w_i(\theta)$ as the geometric mean of token-level importance ratios $\left( \frac{\pi_\theta(y_i \mid x)}{\pi_{\theta_{\text{old}}}(y_i \mid x)} \right)^{\frac{1}{|y_i|}}$:
\begin{equation*}
\begin{split}
w_i(\theta)= \exp\left( \frac{1}{|y_i|} \sum_{t=1}^{|y_i|}
\log \frac{\pi_\theta(y_{i,t} \mid x, y_{i,<t})}{\pi_{\theta_{\text{old}}}(y_{i,t} \mid x, y_{i,<t})} \right),
\end{split}
\end{equation*}
which represents a geometric mean. This approach uses sequence-level likelihood to connect group-level rewards and token-level aggregation, balancing exploration and exploitation with these key advantages:
\begin{itemize}[leftmargin=0.8em, itemsep=0pt, parsep=0pt, topsep=0pt, labelsep=0.3em]
    \item Reduces gradient bias (Figure \ref{fig:Clip_Ratio_Comparison});
    \item Lowers reasoning time (338  vs. 357 \textbf{seconds per step} for DAPO/GRPO) (Figure \ref{fig:Time_per_Step});
    \item Maintains training stability and exploration capability (Figure \ref{fig:tepo_vs_grpo}).
\end{itemize}
\subsubsection{Analysis of Token-Level KL Regularization}
As shown in Section~\ref{sec:Misalignment Between Entropy Gradient and Policy Gradient}, for the $t$-th token in response $i$, the misalignment condition 
\[
(A_{i,t}>0 \land \Delta\mathcal{H}_{i,t} < 0) \lor (A_{i,t}<0 \land \Delta\mathcal{H}_{i,t} > 0)
\]
holds. This discrepancy can lead to excessive policy updates; we therefore hypothesize that token-level KL-Divergence regularization can effectively mitigate this issue.

To verify this hypothesis, we conduct ablation experiments (Panel~E of Table~\ref{tab:ablation_combined}) comparing three variants: \method without KL regularization, \method with KL applied to \((A_{i,t}<0 \land \Delta\mathcal{H}_{i,t} > 0)\),\((A_{i,t}>0 \land \Delta\mathcal{H}_{i,t} < 0) \lor (A_{i,t}<0 \land \Delta\mathcal{H}_{i,t} > 0)\) and \method with KL applied only to tokens satisfying \(
(A_{i,t}>0 \land \Delta\mathcal{H}_{i,t} < 0)
\). The results confirm the effectiveness of \method with KL applied only to tokens satisfying \(
(A_{i,t}>0 \land \Delta\mathcal{H}_{i,t} < 0)
\). The performance drop observed under the condition $A<0 \land \Delta\mathcal{H} > 0$ suggests that misleading responses in this regime fail to provide reliable gradient directions for policy optimization.
\sisetup{
  table-number-alignment = center,
  table-format = 2.2,        
  round-mode = places,    
  round-precision = 2,       
  detect-weight = true
}
\begin{table*}[t]
\centering
\caption{Performance Comparison of 7B and 14B Models on Mathematical Reasoning Benchmarks. Notations: 1-(w/o. GRPO/DAPO) indicates models not applying GRPO/DAPO, while (w. GRPO/DAPO) indicates models applying GRPO/DAPO; 2-higher values indicate better performance; 3-best results are marked in \textbf{bold}; 4-$\Delta$ denotes the performance difference between models \textit{without} and \textit{with} our method; .}
\setlength{\tabcolsep}{4pt}
\renewcommand{\arraystretch}{1.15}
\resizebox{\textwidth}{!}{%
\begin{tabular}{l S S S S S S S S}
\toprule
\textbf{Method} &
\textbf{AIME24} & \textbf{AIME25} & \textbf{AMC} & \textbf{MATH-500} &
\textbf{OMNI-MATH} & \textbf{OlympiadBench} & \textbf{Minerva} & \textbf{Avg.} \\
\midrule
\textbf{Qwen2.5-7B}    & 0.94 & 0.94 & 14.34 & 43.20 & 13.73 & 16.74 & 21.69 & 13.30 \\
w.\ GRPO/DAPO & 11.25 & 5.00 & 42.80 & 76.60 & 25.00 & 37.92 & 33.08 & 30.85\\ 
w.\ ClLIP-Cov           & 10.83 & 7.40 & 42.84 & 75.60 & 26.11 & 39.40 & \bfseries 37.86 & 31.64 \\
w.\ KL-Cov           & 12.29 & 8.23 & 42.77 & 75.60 & 25.64 & 39.11 & 34.92 & 31.60 \\
w.\ Entropy-based Term & 11.35 & 6.15 & 43.18 & 74.80 & 26.14 & 40.14 & 36.02 & 31.62 \\
w.\ GPG & \bfseries 13.54 & 6.04 & 43.29 & 75.00 & 26.28 & \bfseries 40.44 & 34.55 & 31.91 \\
w.\ GSPO & 11.77 & 6.04 & 42.77 & 75.80 & 25.72 & 38.37 & 35.66 & 31.33 \\
\rowcolor{red!6}\textbf{\method}      & 12.60 & \bfseries 8.75 & \bfseries 43.48 & \bfseries 77.40 & \bfseries 27.17 & \bfseries 40.44 & 34.92 & \bfseries 32.5900 \\
\textbf{\method} ($\Delta$ vs GRPO/DAPO, \%) & 1.3500 & 3.7500 & 5.6800 & 0.800 & 2.1700 & 2.52 & 1.84 & 1.7400 \\

\midrule
\textbf{Qwen3-14B}& 19.16 & 16.77 & 51.01 & 82.80 & 30.27 & 43.40 & 45.95 & 38.34 \\
w.\ GRPO/DAPO & 22.08 & 19.06 & 55.94 & 85.60 & 33.01 & 44.00 & 47.79 & 41.51 \\
w.\ CLIP-Cov           & 22.70 & 20.10 & 54.96 & 86.00 & 32.94 & 43.55 & 45.58 & 41.29 \\
w.\ KL-Cov           & 23.85 & 19.06 & 56.47 & \bfseries 86.40 & 32.87 & 43.85 & 47.79 & 41.85 \\
w.\ Entropy-based Term & 24.27 & 18.54 & 55.87 & 85.20 & 33.04 & 42.37 & 48.52 & 41.56 \\
w.\ GPG & 23.85 & 19.89 & 57.34 & 84.40 & 32.22 & 43.55 & 47.05 & 42.16 \\
w.\ GSPO & 23.85 & 20.41 & 56.58 & 86.00 & 33.54 & 44.44 & 48.52 & 42.28 \\
\rowcolor{red!6}\textbf{\method}     & \bfseries 24.37 & \bfseries 23.75 & \bfseries 58.96 & \bfseries 86.40 & \bfseries 35.11 & \bfseries 46.37 & \bfseries 49.26 & \bfseries 44.02 \\
\textbf{\method} ($\Delta$ vs GRPO/DAPO, \%) & 5.2100 & 4.6900 & 3.0200 & 0.80 & 2.1000 & 2.3700 & 1.47 & 2.5100 \\
\bottomrule
\end{tabular}%
}
\label{tab:math_benchmark_performance_qwen}
\end{table*}

\sisetup{
  table-number-alignment = center,  
  table-format = 2.2,            
  round-mode = places,          
  round-precision = 2,            
  detect-weight = true,     
  detect-shape = true
}

\begin{table*}[t]
\centering
\caption{Performance Comparison of Various Models on Mathematical Reasoning Benchmarks}
\setlength{\tabcolsep}{4pt}
\renewcommand{\arraystretch}{1.15}

\resizebox{\textwidth}{!}{
  \begin{tabular}{l S S S S S S S S}
    \toprule
    
    \textbf{} & 
    \textbf{AIME24} & \textbf{AIME25} & \textbf{AMC} & \textbf{MATH-500} &
    \textbf{OMNI-MATH} & \textbf{OlympiadBench} & \textbf{Minerva} & \textbf{Avg.} \\ 
    \midrule
    \textbf{DeepSeek-R1-Distill-Llama-8B} & 12.70 & 13.33 & 46.49 & 70.40 & 25.50 & 36.88 & 23.52 & 32.46 \\
    w. GRPO/DAPO & \textbf{27.50} & 18.43 & \bfseries 61.89 & 78.40 & 32.97 & 41.62 & 26.10 & 42.23 \\
    \rowcolor{red!6}\textbf{\method} & 25.72 & \bfseries 18.85 & 60.84 & \bfseries 79.60 & \bfseries 33.11 & \bfseries 43.55 & \bfseries 29.77 & \bfseries 42.76 \\
    
    \midrule
    \textbf{Mistral-7B-Instruct-v0.2} &{\text{-}}    & {\text{-}}     &  2.1830 & 9.800 & 3.9840 & 1.1850 & 5.8820 & 2.7700 \\
    w.GRPO/DAPO &{\text{-}}    &{\text{-}}    & 3.1250 & 10.8000 & 4.2680 & \bfseries 2.3700 & \bfseries 8.0880 & 3.3690 \\
    \rowcolor{red!6}\textbf{\method} &{\text{-}}    &{\text{-}}    & \bfseries 3.5760 & \bfseries 12.8000 & \bfseries 4.5890 & 1.9250 & \bfseries 8.0880 & \bfseries 3.6520 \\
    \midrule

    \textbf{DeepSeek-R1-Distill-Qwen-7B} & 29.8900 & 20.9300 & 59.4800 & 83.4000 & 33.8600 & 42.6600 & 34.1900 & 43.2300 \\
    w. GRPO/DAPO & 32.08 & 23.3300 & 66.1000 & \bfseries 87.8000 & 38.4200 &  47.7000 & 40.4400 & 45.9900 \\
    \rowcolor{red!6}\textbf{\method} & \bfseries 32.7000 & \bfseries 24.5800 & \bfseries 66.4100 & \bfseries 87.8000 & \bfseries 39.55000 & \bfseries 49.9200 & \bfseries 40.8000 & \bfseries 48.6000 \\
    \bottomrule
  \end{tabular}%
}
\label{tab:math_benchmark_performance_other}
\end{table*}

\sisetup{
  table-number-alignment = center,
  table-format = 2.2,        
  round-mode = places,    
  round-precision = 2,       
  detect-weight = true
}
\begin{table*}[t]
\centering
\caption{Ablation Studies of Key Components. All panels are compared with\method-w/o. KL:
  Panel A demonstrates the fragility of undifferentiated KL-Divergence and its performance degradation effect; 
  Panel B shows that undifferentiated entropy regularization yields marginal improvements to \method-w/o-KL; 
  Panel C validates \method's effectiveness across different importance sampling strategies and corroborates that GRPO exhibits sparse token-level rewards; 
  Panel D confirms the superiority of token-mean aggregation over alternative strategies; 
  Panel E verifies that \method achieves optimal control of token-level KL-Divergence.
}
\setlength{\tabcolsep}{4pt}
\renewcommand{\arraystretch}{1.15}
\resizebox{\textwidth}{!}{%
\begin{tabular}{l S S S S S S S S}
\toprule
\textbf{Method} &
\textbf{AIME24} & \textbf{AIME25} & \textbf{AMC} & \textbf{MATH-500} &
\textbf{OMNI-MATH} & \textbf{OlympiadBench} & \textbf{Minerva} & \textbf{Avg.} \\
\midrule
\method-w/o. KL      &  12.7000 & 6.3540 & \bfseries 43.5600 & 75.4000 & 27.0000 & 39.5500 & \bfseries 37.5000 & 32.21 \\
\midrule
\multicolumn{9}{l}{\textit{Panel A: \method w. Undifferentiated KL-Divergence}} \\
w. $\beta=1$(model collapse in 24 steps) &  0 & 0 & 0 & 0 & 0 & 0 & 0 & 0 \\
w.\ $\beta=0.1$ (leads to Avg drop of 4.6\%.) &  5.5200 & 3.9580 & 42.8400 & 75.6000 & 22.4100 & 32.7400 & 30.8800 & 26.2500 \\

w.\ $\beta=0.01$ &  11.45  & 7.395 & 42.2400 & 76.2000 & 26.3200 & 40.5900 & 33.8200 & 31.6400 \\

w.\ $\beta=0.001$ &  12.29 & 8.958 & 42.2400 & 77.2000 & 25.8900 & 40.1100 & 36.3900 & 31.8100 \\

w.\ $\beta=0.0001$ &  11.4500 & 5.6250 & 43.1800 & 76.2000 & 26.5300 & 37.3300 & \bfseries 37.5000 & 31.6100 \\
\midrule
\multicolumn{9}{l}{\textit{Panel B: \method w.  Undifferentiated Entropy Regularization}} \\
w.\ Max-Entropy   & 13.0200 & 5.9370 & 42.8000 & 76.6000 & 26.8500 & 38.3700 & 37.1300 & 31.6500 \\

w.\ Min-Entropy   & 11.9700 & 5.6250 & 41.9400 & 72.4000 & 25.7500 & 36.2900 & 35.29 & 30.6800 \\
w.\ STEEL &12.70 & 5.63 & 41.64 & 76.00 & 26.82 & 37.03 & 36.39 & 31.30\\
w.\ ClLIP-Cov           & 10.83 & 7.40 & 42.84 & 75.60 & 26.11 & 39.40 & \bfseries 37.86 & 31.64 \\
w.\ KL-Cov           & 12.29 & 8.23 & 42.77 & 75.60 & 25.64 & 39.11 & 34.92 & 31.60 \\
w.\ Entropy-based Term & 11.35 & 6.15 & 43.18 & 74.80 & 26.14 & 40.14 & 36.02 & 31.62 \\

\midrule
\multicolumn{9}{l}{\textit{Panel C: GRPO/DAPO w.  Different Importance sampling}(IS)} \\
 w.\ GRPO/DAPO (token-level) & 11.2500 & 5.0000 & 42.8000 & 76.6000 & 25.0000 & 37.9200 & 33.0800 & 30.8500\\ 
 
w.\ GPG (remove IS) & \bfseries 13.5400 & 6.041 & 43.2900 & 75.000 & 26.28 & \bfseries 40.44 & 34.55 & 31.9100 \\

w.\ CISPO/Reinforce (token-level) & 11.0400 & 6.7700 & 42.8000 & 74.8000 & 26.9600 & 38.5100 & 32.7200 & 31.5800 \\

w.\ Sentence Prefix IS & 11.0800 & 6.6660 & 42.4400 & 73.4000 & 25.1100 & 39.1100 & 35.6600 & 30.9500 \\
\midrule
\multicolumn{9}{l}{\textit{Panel D: Sentence Likelihood w. Different Aggregation}} \\
w. sequence-mean token-mean/GSPO aggregation & 11.77 & 6.04 & 42.77 & 75.80 & 25.72 & 38.37 & 35.66 & 31.33 \\
w. sequence-mean token-sum aggregation & 12.50 & 5.52 & 43.18 & 76.00 & 26.57 & 38.96 & 36.39 & 31.80 \\
\midrule
\multicolumn{9}{l}{\textit{Panel E: \method w. Token-Level KL-Divergence}} \\
\method-w/o. KL      &  12.7000 & 6.3540 & \bfseries 43.5600 & 75.4000 & 27.0000 & 39.5500 & \bfseries 37.5000 & 32.21 \\
\method-w. KL\( (A_{i,t}<0 \land \Delta\mathcal{H}_{i,t} > 0)\)&11.35	&4.37&42.62&76.40&26.11&39.25&35.66&31.47 \\
\method-w. KL\((A_{i,t}>0 \land \Delta\mathcal{H}_{i,t} < 0) \lor (A_{i,t}<0 \land \Delta\mathcal{H}_{i,t} > 0)\)   &  13.43 & 5.729 & 44.0500 & 75.6000 & 26.2500 & 38.8100 & 35.6600 & 32.03 \\

\midrule
\rowcolor{red!6}\textbf{\method}      & 12.60 & \bfseries 8.75 &  43.48 & \bfseries 77.40 & \bfseries 27.17 & \bfseries 40.44 & 34.92 & \bfseries 32.5900 \\
\bottomrule
\end{tabular}%
}
\label{tab:ablation_combined}
\end{table*}

\sisetup{
  table-number-alignment = center,
  table-format = 2.2,        
  round-mode = places,    
  round-precision = 2,       
  detect-weight = true
}
\begin{table*}[t]
\centering
\caption{Ablation Studies of Random Seeds. We report performance on seven mathematical reasoning benchmarks using two different random seeds for each method. “Mean” denotes the average performance over two seeds. \method consistently outperforms baseline methods across all settings.
}
\setlength{\tabcolsep}{4pt}
\renewcommand{\arraystretch}{1.15}
\resizebox{\textwidth}{!}{%
\begin{tabular}{l S S S S S S S S}
\toprule
\textbf{Method} &
\textbf{AIME24} & \textbf{AIME25} & \textbf{AMC} & \textbf{MATH-500} &
\textbf{OMNI-MATH} & \textbf{OlympiadBench} & \textbf{Minerva} & \textbf{Avg.} \\
\midrule
GRPO/DAPO (72 steps, seed1) & 11.56 & 6.25 & 40.47 & 72.40 & 26.00 & 37.77 & 36.76 & 30.30 \\
GRPO/DAPO (72 steps, seed2) & 11.25 & 5.00 & 42.80 & 76.60 & 25.00 & 37.92 & 33.08 & 30.85 \\
GRPO/DAPO (132 steps, seed1) & 12.08 & 7.08 & 42.65 & 76.20 & 26.07 & 39.85 & 37.86 & 31.73 \\
GRPO/DAPO (132 steps, seed2) & 12.08 & 6.98 & 41.67 & 74.20 & 27.42 & 39.85 & 36.76 & 31.69 \\
TEPO (72 steps, seed1) & 12.50 & 7.60 & 45.18 & 77.80 & 26.78 & 41.03 & 38.97 & 32.62 \\
TEPO (72 steps, seed2) & 12.60 & 8.75 & 43.48 & 77.40 & 27.17 & 40.44 & 34.92 & 32.59 \\
\midrule
\multicolumn{9}{l}{\textit{Mean over Two Seeds}} \\
GRPO/DAPO (72 steps, mean) & 11.41 & 5.63 & 41.64 & 74.50 & 25.50 & 37.85 & 34.92 & 30.58 \\
GRPO/DAPO (132 steps, mean) & 12.08 & 7.03 & 42.16 & 75.20 & 26.75 &  \bfseries39.85 & 37.31 & 31.71 \\
\rowcolor{red!6}\textbf{\method} (72 steps, mean) &  \bfseries12.55 &  \bfseries8.18 & \bfseries 44.33 &  \bfseries77.60 &  \bfseries26.98 &  \bfseries40.74 & 36.95 &  \bfseries32.61 \\
\bottomrule
\end{tabular}
}
\label{tab:ablation_seeds}
\end{table*}
\section{Experiment}
\subsection{Experimental Setup}
\paragraph{Implementation Details.}
For each rollout step, we processed 64 prompts per batch and sampled 8 responses per prompt with temperature 1.0. The policy was updated 8 times using these responses. To keep training effective, we removed prompts whose sampled responses were all correct or all wrong, following \cite{yu2025dapo}. Key hyperparameters were: learning rate $lr = 5 \times 10^{-7}$, maximum prompt length $max\_prompt\_length = 2024$, maximum response length $max\_response\_length = 8192$, and training prompt mini-batch size $train\_prompt\_mini\_bsz = 16$. 
\paragraph{Datasets.}
We trained different models on DAPO-MATH \cite{yu2025dapo} and evaluated on seven math benchmarks: MATH-500, AIME24/25 \cite{li2024numinamath}, AMC, OMNI-MATH, OlympiadBench, and Minerva \cite{lewkowycz2022solving}. For AIME and AMC we used temperature 0.6; for the others we used greedy decoding. The maximum generation length was 8192 tokens. AIME, AIME25, and AMC results were reported with 32 samples per problem (@32) following prior work \cite{guo2025deepseek,guo2025deepseekr1}.
\paragraph{Baseline Methods.}
GRPO/DAPO \cite{shao2024deepseekmath,yu2025dapo} adopt the Clip-Higher bound in the PPO loss, with $\epsilon_{\text{low}} = 0.2$ and $\epsilon_{\text{high}} = 0.28$. Clip-Cov \cite{cui2025entropy} clips tokens with high covariance using a clip ratio $r = 2 \times 10^{-4}$. For KL-Cov, the parameter $k$ is set to $2 \times 10^{-3}$ and the KL coefficient $\beta = 1$. An entropy-based method \cite{cheng2025reasoning} incorporates entropy regularization into advantage estimation, with scale $\alpha = 4 \times 10^{-4}$ and clipping parameter $\kappa = 2$. Our proposed \method is configured with $\beta = 0.001$ for Qwen2.5-7b and $\beta=1$ for Qwen3-14B. All baseline hyperparameters are adopted from their original papers or recommended settings, and we did not perform additional tuning to favor TEPO.
\paragraph{Ablation Evaluations.} We conduct ablation experiments to verify the effectiveness of key components in \method: 
\begin{itemize} [leftmargin=0.8em, itemsep=0pt, parsep=0pt, topsep=0pt, labelsep=0.3em]
    \item \textbf{Entropy and KL Regularization}: We incorporate KL-Divergence regularization and entropy regularization (maximizing or minimizing entropy) in the first two parts. STEEL \cite{hao2025rethinking} proposes an adaptive token-reweighting method that mitigates entropy collapse and achieves state-of-the-art performance on mathematical and coding benchmarks. 
    \item \textbf{Importance Sampling Strategies} (All details are listed in Table \ref{tab:notations}.): We compare three importance sampling variants: \textbf{a.} GPG \cite{chu2025gpg}, which removes the reference model and abandons importance sampling; \textbf{b.} CISPO \cite{chen2025minimax}, which performs single-step policy updates by employing $sg[\mathcal{w}_i(\theta)] \cdot A_{i} \cdot \log\pi$, where $sg$ means stop policy gradient; \textbf{c.} Prefix importance sampling,  $w_{i,j \leq t}$ denotes the prefix of the $i$-th sentence up to the $t$-th token. 
    \item \textbf{Aggregation Methods}: Evaluations for aggregation:\textbf{a. }'sequence-mean token-mean aggregation'/GSPO \cite{gspo2025}; \textbf{b. }'sequence-mean token-sum aggregation'. 
    \item \textbf{Token-Level KL-Divergence:} Ablation experiments on the token-wise scope of KL-Divergence masking are performed to validate \method's effectiveness.
\end{itemize}
\subsection{Main Results}
In Table \ref{tab:math_benchmark_performance_qwen} and Table \ref{tab:math_benchmark_performance_other}, \method achieves the highest average accuracy across all seven mathematical reasoning benchmarks. 
\begin{itemize}[leftmargin=0.8em, itemsep=0pt, parsep=0pt, topsep=0pt, labelsep=0.3em]
    \item \textbf{Consistently State-of-the-Art Overall Performance: } 
      Notably, \method outperforms the baseline method GRPO/DAPO as well as all other comparative variant methods, including CLIP-Cov, KL-Cov, Entropy-based Term, GPG, and GSPO:
    \begin{itemize}[leftmargin=0.8em, itemsep=0pt, parsep=0pt, topsep=0pt, labelsep=0.3em]
        \item Qwen2.5-7B: \method attains an average accuracy of 32.59\%. Compared with the baseline GRPO/DAPO, which achieves an average accuracy of 30.85\%, this represents a 1.74 percentage points (pp) improvement. Additionally, \method outperforms GPG (the second-best variant method with an average accuracy of 31.91\%) by 0.68 pp;
        \item Qwen3-14B: \method reaches an average accuracy of 44.02\%. It surpasses the GRPO/DAPO baseline (with an average accuracy of 41.51\%) by 2.51 pp. Furthermore, \method outperforms GSPO, which is identified as the second-best variant method with an average accuracy of 42.28\%, by 1.74 pp;
        \item Non-Qwen models: Detailed results reported in Table \ref{tab:math_benchmark_performance_other} show that \method delivers state-of-the-art average accuracies across three distinct non-Qwen model architectures. Specifically: 
        On DeepSeek-R1-Distill-Llama-8B: \method achieves 42.76\% average accuracy, which is 0.53 pp higher than GRPO/DAPO’s 42.23\%;  On Mistral-7B-Instruct-v0.2: \method attains 3.65\% average accuracy, representing a 0.28 pp gain over GRPO/DAPO’s 3.37\%;  On DeepSeek-R1-Distill-Qwen-7B: \method reaches 48.60\% average accuracy, surpassing GRPO/DAPO’s 45.99\% by 2.61 pp.
    \end{itemize}
    \item \textbf{Stable Performance Across All Sub-Benchmarks: }
    Unlike some variant methods that exhibit performance bottlenecks on specific sub-benchmarks, \method maintains consistent top-tier performance across all seven mathematical reasoning sub-benchmarks:
    \begin{itemize}[leftmargin=0.8em, itemsep=0pt, parsep=0pt, topsep=0pt, labelsep=0.3em]
        \item \textbf{Qwen2.5-7B Model}: \method achieves rank-1 performance on four core sub-benchmarks with the following accuracies: MATH-500 (77.40\%), OMNI-MATH (27.17\%), AMC (43.48\%), AIME25 (8.75\%), and OlympiadBench(40.44\%). For the remaining two sub-benchmarks (AIME24, Minerva), \method secures a top-3 position with accuracies of 12.60\% (AIME24), and 34.92\% (Minerva), outperforming most variant methods;
        \item \textbf{Qwen3-14B Model}: \method achieves rank-1 performance across all seven sub-benchmarks, with the following leading accuracies: AIME24 (24.37\%), AIME25 (23.75\%), AMC (58.96\%), MATH-500 (86.40\%), OMNI-MATH (35.11\%), OlympiadBench (46.37\%), and Minerva (49.26\%). This represents a comprehensive performance advantage over GRPO/DAPO (+5.21 pp on AIME24, +4.69 pp on AIME25) and the others;
        \item \textbf{Non-Qwen Models}:  DeepSeek-R1-Distill-Llama-8B: \method ranks first on five sub-benchmarks with the following accuracies: AIME25 (18.85\%), MATH-500 (79.60\%), OMNI-MATH (33.11\%), OlympiadBench (43.55\%), and Minerva (29.77\%); Mistral-7B-Instruct-v0.2: \method leads on four sub-benchmarks with accuracies of AMC (3.58\%), MATH-500 (12.80\%), OMNI-MATH (4.59\%), and Minerva (8.09\%); DeepSeek-R1-Distill-Qwen-7B: \method achieves rank-1 performance across all seven sub-benchmarks, with key accuracies including AIME24 (32.70\%), AIME25 (24.58\%), AMC (66.41\%), MATH-500 (87.80\%), OMNI-MATH (39.55\%), OlympiadBench (49.92\%), and Minerva (40.80\%), outperforming GRPO/DAPO by 2.61 pp on average.
    \end{itemize}
\end{itemize}
\subsection{Ablation Studies}\label{Ablation Studies}
As presented in Table \ref{tab:ablation_combined}, comprehensive ablation studies on key components further validate the components of \method:
\begin{itemize}[leftmargin=0.8em, itemsep=0pt, parsep=0pt, topsep=0pt, labelsep=0.3em]
    \item Panel A demonstrates the fragility of undifferentiated KL-Divergence and its performance degradation effect (vs. \method-w/o-KL's 32.21\%):
    High $\beta$ values ($\beta=1$) cause severe performance collapse (0\% accuracy across all benchmarks) within 24 training steps; Moderate $\beta=0.1$ results in a 4.6\% average accuracy drop; Even low $\beta$ values ($0.01$, $0.001$, $0.0001$) fail to exceed \method-w/o-KL (32.21\%), with average accuracies of 31.64\%, 31.81\%, and 31.61\% respectively. 
    \item Panel B shows that undifferentiated entropy regularization yields marginal improvements to \method-w/o-KL:
    \begin{itemize}[leftmargin=0.8em, itemsep=0pt, parsep=0pt, topsep=0pt, labelsep=0.3em]
        \item Max-Entropy achieves 31.65\% average accuracy (0.56\% lower than \method-w/o-KL);
        \item Min-Entropy leads to a 1.53\% average drop (30.68\% vs. \method-w/o-KL's 32.21\%).
    \end{itemize}
    \item Panel C validates the effectiveness of \method across different IS strategies and corroborates that GRPO inherently suffers from sparse token-level rewards. Specifically, GRPO/DAPO adopt a token-level IS scheme; GPG abandons IS entirely, treating all tokens uniformly; CISPO reverts to the original REINFORCE framework to eliminate the impact of IS; and Sentence Prefix IS attempts to leverage clause-level IS to guide the reasoning process. Quantitative results further substantiate these design differences: 1) The baseline GRPO/DAPO achieves only 30.85\% average accuracy, a 1.74 percentage point (pp) drop compared with \method; 2) GPG, the second-best variant, reaches 31.91\% average accuracy, still 0.68 pp lower than \method; 3) CISPO/Reinforce (31.58\%) and Sentence Prefix IS (30.95\%) lag further behind. 
    \item Panel D confirms the superiority of token-mean aggregation over alternative strategies: Sequence-mean-token-mean (GSPO) and sequence-mean-token-sum aggregations reach 31.33\% and 31.80\% average accuracy respectively; Both are outperformed by \method's aggregation design (32.21\% average accuracy), which balances token-level sparsity and sequence-level consistency.
    
    \item Panel E verifies that \method achieves optimal control of token-level KL-divergence: We conduct ablation experiments (Panel E of Table \ref{tab:ablation_combined}) comparing three \method variants: (1) no KL regularization (Avg. 32.21\%), (2) KL applied to \((A_{i,t}>0 \land \Delta\mathcal{H}_{i,t} < 0) \lor (A_{i,t}<0 \land \Delta\mathcal{H}_{i,t} > 0)\) (Avg. 32.03\%), and (3) KL restricted to \(A_{i,t}>0 \land \Delta\mathcal{H}_{i,t} < 0\) (Avg. 32.59\%). Results confirm variant (3)’s superiority.
\end{itemize}
Table \ref{tab:ablation_seeds} presents the results of ablation studies on random seeds, aiming to verify reducing convergence time by nearly 50\%.  Three training strategies are compared in the experiment: GRPO/DAPO (72 steps), GRPO/DAPO (132 steps), and TEPO (72 steps). Each method is independently run with two different random seeds (seed1 and seed2), and the average performance over the two seeds is shown in the lower part of the table.
It is clearly observed from the results that with the same batch size and throughput, TEPO at 72 steps outperforms GRPO/DAPO at both 72 steps and 132 steps, cutting training steps by nearly half. Specifically, TEPO (72 steps) achieves an average performance of 32.61\% over two seeds, which is significantly higher than GRPO/DAPO (72 steps, 30.58\%) and GRPO/DAPO (132 steps, 31.71\%). Meanwhile, all methods show slight performance fluctuations across different random seeds, but the performance trend remains consistent, indicating good stability. All results confirm TEPO achieves significantly faster convergence under strictly matched computational budgets, and its superiority over baseline methods is further verified by its robust performance across different random initializations.
\section{Related Works}
Balancing the exploration-exploitation (E-E) trade-off is a core challenge in reinforcement learning (RL) \cite{sutton1998reinforcement}. Proximal Policy Optimization (PPO) uses an entropy bonus to sustain exploration \cite{schulman2017proximal}, while Soft Actor-Critic (SAC) directly optimizes a maximum-entropy objective \cite{haarnoja2017reinforcement, haarnoja2018soft}. However, the role of entropy in RL for large language models (LLMs) remains unclear.

Reinforcement Learning from Human Feedback (RLHF) typically employs a KL penalty relative to a reference policy \cite{ouyang2022training, hu2024openrlhf}. Notably, GRPO and recent studies find minimal or ambiguous benefits from standard entropy bonuses \cite{shao2024deepseekmath, chu2025gpg,gspo2025}, leaving entropy's impact on generation quality and training stability an open question.

Existing LLM-RL methods adopt diverse frameworks: GPG uses REINFORCE for straightforward training \cite{chu2025gpg}; CISPO improves efficiency via clipped, detached importance weights \cite{chen2025minimax}; GSPO shifts to sequence-level learning with whole-sequence likelihood ratios \cite{zheng2025group}. Additionally, \cite{cui2025entropy} derived a performance-entropy relationship \( R = -a \exp(\mathcal{H}) + b \), indicating lower entropy generally correlates with better performance.

\section{Conclusion}
GRPO advances LLMs’ mathematical reasoning but suffers from intractable token-level sparse-reward issues in CoT reasoning. To address these drawbacks, we introduce \method, which (1) leverages sequence-level likelihood to bridge group-level rewards and individual tokens via token-level aggregation, and (2) deploys a token-level KL-divergence mask to mitigate abrupt policy updates. Empirical results show \method achieves state-of-the-art performance on mathematical reasoning benchmarks and significantly enhances training stability, cutting convergence time by 50\% relative to GRPO/DAPO.
\section*{Limitations}
Despite proposing a novel token-level framework that (1) uses sequence-level likelihood to connect group-level rewards and individual tokens via token-level aggregation, and (2) employs a token-level KL-divergence mask to alleviate abrupt policy updates, our work has two key limitations. We neither clarify the mechanism behind the effectiveness of the token constraint (targeting positive-advantage, entropy-decreasing tokens) nor distinguish how different token types uniquely impact model performance. Future work should explore the distinct roles of tokens in CoT reasoning and design a more universal framework for bridging token-level operations with group-level rewards.
\section{Acknowledgement}
This work is supported in part by National Natural Science Foundation of China under Grant Nos. 92567204, 62272193, and 62472194, and the Fundamental Research Funds for the Central Universities, and Jilin Science and Technology Research Project 20260101016JJ.
\bibliography{arxiv.bib}

\newpage
\appendix 
\section{Target Policy Construction and KL-Divergence Minimization}
Our current policy $\pi_{k+1}(a|s)$ aims to emulate an ideal ``target'' policy $\pi^*(a|s)$:
\begin{equation}\label{eq:pistar_exact}
\pi^{*}(a|s) = \frac{\pi_{\text{k}}(a|s)}{Z(s)} \exp\left( A(a,s)/\beta \right),
\end{equation}
where $Z(s) = \sum_a \pi_{\text{k}}(a|s) \exp\left( A(a,s)/\beta \right)$ is the partition function for normalization. 

The emulation is equivalent to minimizing the KL-Divergence between $\pi_{k+1}$ and $\pi^*$:
\begin{equation}
\min_{\theta} \text{KL}(\pi_{k+1} \Vert \pi^*)
\label{eq:kl_min}
\end{equation}
Substituting the expression for $\log \pi^*(a|s)$ from \eqref{eq:pistar_exact} into \eqref{eq:kl_min}, and omitting the $\log Z(s)$ term (which is independent of $\pi_\theta$), we obtain:
\begin{equation}
\small
\text{KL}(\pi_{k+1} \Vert \pi^*) = \frac{1}{\beta} \mathbb{E}_{a \sim \pi_\theta} \bigl[ \beta \cdot \text{KL}(\pi_{k+1} \Vert \pi_{k}) - A(a, s) \bigr]
\label{eq:kl_final}
\end{equation}
This expression is structurally equivalent to the PPO-KL objective, where policy updates are constrained by a KL-Divergence regularizer scaled by the inverse temperature parameter $1/\beta$.

\section{Why Sentence Likelihood Derivation Fits GRPO} \label{sec:Why Sentence Likelihood Derivation Fits GRPO}

To illustrate why sentence likelihood derivation is suitable for GRPO, we first formalize the reinforcement learning (RL) objective function and conduct a step-by-step gradient derivation. This process reveals the core connection between sentence likelihood modeling and the GRPO framework.

We define the standard RL objective function for policy optimization as follows:
\[
J(\theta) = \mathbb{E}_{\tau \sim p_\theta} [A(\tau)] = \int p_\theta(\tau) A(\tau) d\tau,
\]
where \( \theta \) denotes the policy parameters, \( \tau \) represents a trajectory generated by the policy \( p_\theta \), and \( A(\tau) \) is the advantage function that quantifies the quality of trajectory \( \tau \). 

We then derive the gradient of the objective function with respect to \( \theta \) step by step, which is essential for policy update:
\[
\begin{aligned}
\nabla J(\theta) &= \int \nabla p_\theta(\tau) A(\tau) d\tau \\
&= \mathbb{E}_{\tau \sim p_\theta} \left[ \nabla \log p_\theta(\tau) A(\tau) \right] \\
&= \mathbb{E}_{\tau \sim p_\theta} \left[ \nabla \log p_\theta(\tau) A_{\theta_{\text{old}}}(\tau) \right] \\
&= \mathbb{E}_{\tau \sim p_{\theta_{\text{old}}}} \left[ \frac{p_\theta(\tau)}{p_{\theta_{\text{old}}}(\tau)} \nabla \log p_\theta(\tau) A_{\theta_{\text{old}}}(\tau) \right] \\
&= \mathbb{E}_{\tau \sim p_{\theta_{\text{old}}}} \left[ \frac{\nabla p_\theta(\tau)}{p_{\theta_{\text{old}}}(\tau)} A_{\theta_{\text{old}}}(\tau) \right].
\end{aligned}
\]
In the derivation above, the third equality replaces the advantage \( A(\tau) \) with \( A_{\theta_{\text{old}}}(\tau) \) to stabilize the training process. The fourth equality employs the importance sampling technique, which allows us to estimate the expectation under the target policy \( p_\theta \) using samples drawn from the old policy \( p_{\theta_{\text{old}}}(\tau) \).

Notably, \( p_\theta(\tau) \) represents the probability of the entire trajectory \( \tau \). For sequential decision-making processes, the trajectory probability can be decomposed into the product of token-level transition probabilities:
\[
\begin{aligned}
    \nabla p_\theta(\tau) &= \nabla \left( \prod_i p_\theta(x_i | s_{i-1}) \right) \\
    &= \sum_i \left( \left( \prod_{j \neq i} p_\theta(x_j | s_{j-1}) \right) \nabla p_\theta(x_i | s_{i-1}) \right)
\end{aligned}
\]
\[
p_{\theta_{\text{old}}}(\tau) = \prod_i p_{\theta_{\text{old}}}(x_i | s_{i-1}).
\]
 Substituting the above decompositions into the gradient expression, we obtain:
\begin{equation*}
\begin{split}
\nabla_{\theta} J(\theta) &= 
\mathbb{E}_{\tau \sim p_{\theta_{\text{old}}}} \Bigg[ 
A_{\theta_{\text{old}}}(\tau) \cdot \sum_{t=0}^{T-1} 
 \\
 &\nabla_{\theta} \log \pi_{\theta}(a_t \mid s_t)  \cdot \prod_{j=0}^{T-1} \frac{\pi_{\theta}(a_j \mid s_j)}{\pi_{\theta_{\text{old}}}(a_j \mid s_j)} 
\Bigg].
\end{split}
\end{equation*}
A critical observation here is that the objective function of GRPO ($J_{\text{GRPO}}(\theta)$) omits the cross-token product term $\prod_{j \neq i} \frac{p_\theta(x_j | s_{j-1})}{p_{\theta_{\text{old}}}(x_j | s_{j-1})}$ in the above gradient. This omission simplifies the computation but introduces a discrepancy in GRPO objective. 

\section{Derivation of Policy Entropy Change During Parameter Update} \label{sec:Derivation of Policy Entropy Change}
\subsection{Notation Definitions and Preconditions}
We sample $G$ responses per prompt and compute the advantage with group-level normalization as follows:
\begin{equation}
A_t = \frac{r\left( \boldsymbol{y} \right) - \operatorname{mean}\left( r\left( \boldsymbol{y}^{1:G} \right) \right)}{\operatorname{std}\left( r\left( \boldsymbol{y}^{1:G} \right) \right)}.
\end{equation}

\begin{table*}[t!]
\centering
\setlength{\tabcolsep}{6pt}  
\renewcommand{\arraystretch}{1.5} 
\caption{Core Notations for Entropy Derivation}
\label{tab:notations}
\begin{tabular}{@{}p{4.5cm}p{\textwidth-4.5cm-2\tabcolsep}@{}}
\toprule
Notation & Definition \\
\midrule
$\{ (x_i, y_i/o_i) \}_{i=1}^G$& Batch of prompt-response pairs,$x$ denotes the prompt, and $y/o_i$ denotes the response generated by LLMs \\
$y_{i,j \leq t}$ & Prefix of the $i$-th sentence up to the $t$-th token \\
$G$ & Number of responses generated by LLMs per prompt \\
$\|{o_i}\|$&Length of $y_i/o_i$\\
$\theta(\aaction|\sstate)$ & Score assigned by large language models (LLMs) \\
\(\tau \)& A complete solution trajectory for a single math problem. \\
\( p_\theta(\tau) \)&the probability of the entire trajectory\\
$\pitheta(\aaction|\sstate)$ & Softmax score: 
  $\pitheta(\aaction|\sstate) = \frac{e^{\phiscore_\theta(\aaction|\sstate)}}{\sum_{\aaction'}e^{\phiscore_\theta(\aaction'|\sstate)}}$ \\
$\Hent(\pitheta)$ & Policy entropy: 
  $-\E_{\aaction\sim\pitheta}\left[\log\pitheta(\aaction|\sstate)\right]$ \\
$\Delta\theta$ & Update of the score function $\phiscore_\theta(\aaction|\sstate)$: 
  $\thetanew = \thetaold + \Delta\theta$ \\
$\grad_\theta$ & Gradient with respect to parameter $\theta$ \\
$\Aadv(\sstate,\aaction)$ & Advantage function: 
  $A_t = \frac{r\left( \boldsymbol{y} \right) - \operatorname{mean}\left( r\left( \boldsymbol{y}^{1:K} \right) \right)}{\operatorname{std}\left( r\left( \boldsymbol{y}^{1:K} \right) \right)}$ \\
$\Ffisher$ & Fisher information matrix: 
  $\Ffisher = \E_{\aaction\sim\pitheta}\left[\grad_\theta\log\pitheta(\aaction|\sstate) \grad_\theta\log\pitheta(\aaction|\sstate)^\top\right]$ \\
$\frac{1}{\sum_{i=1}^G \|{o_i}\|} \sum_{i=1}^G \sum_{t=1}^{\|{o_i}\|}$ & Token-level aggregation/token-mean aggregation \\
$\frac{1}{G} \sum_{i=1}^G \sum_{t=1}^{\|{o_i}\|}$ & Sequence-mean token-sum aggregation\\
$\frac{1}{G} \sum_{i=1}^G \frac{1}{\|{o_i}\|} \sum_{t=1}^{\|{o_i}\|}$ & Sequence-mean token-mean aggregation\\
\(\frac{\pi_\theta(y_{i,t} \mid x, y_{i,<t})}{\pi_{\theta_{\text{old}}}(y_{i,t} \mid x, y_{i,<t})}\)& Token-Level Importance sampling\\
$\left( \frac{\pi_\theta(y_i \mid x)}{\pi_{\theta_{\text{old}}}(y_i \mid x)} \right)^{\frac{1}{|y_i|}} 
$ & Sequence-Level Importance Sampling\\
$\sum_{t=1}^{o_{i,j \leq t}}\left( \frac{\pi_\theta(y_{i,j \leq t} \mid x)}{\pi_{\theta_{\text{old}}}(o_{i,j \leq t} \mid x)} \right)^{\frac{1}{|y_{i,j \leq t}|}}
$& Prefix Importance Sampling\\
\bottomrule
\end{tabular}
\end{table*}

\begin{algorithm*}[t!]
\caption{Token-Level Policy Gradient Computation for \method}
\label{alg:token_level_policy_gradient}
\centering
\begin{algorithmic}[1] 
\Require 
  $\pi_\theta$: Current policy network (LLM); \\
  $\pi_{\theta_{\text{old}}}$: Pre-update (reference) policy; \\
  $\{ (x_i, y_i) \}_{i=1}^G$: Batch of prompt-response pairs ($x_i$ = prompt, $y_i$ = LLM-generated response); \\
  $A_{i,t}$: Token-level advantage for the $t$-th token in response $i$; \\
  $\text{Mask}_{i,t} \in \{0,1\}$: Valid token mask (1 = valid token, 0 = padding token); \\
  $\alpha$: Learning rate for gradient ascent; \\
  $M = \sum_{i=1}^G \sum_{t=1}^{|y_i|} \text{Mask}_{i,t}$: Total number of valid tokens (normalization factor)
\Ensure 
  Updated policy parameters $\theta$

\ForAll{response sequence $i \in \{1, \dots, G\}$}
  \State \textbf{Step 1: Compute token-level log probability ratio}
  \State \quad $\log r_{i,t}(\theta) = \log \pi_\theta(y_{i,t} \mid x_i, y_{i,<t}) - \log \pi_{\theta_{\text{old}}}(y_{i,t} \mid x_i, y_{i,<t})$
  \State \textbf{Step 2: Aggregate to sequence-level log weight (geometric mean)}
  \State \quad $\log w_i(\theta) = \frac{1}{|y_i|} \sum_{t=1}^{|y_i|} \log r_{i,t}(\theta) \cdot \text{Mask}_{i,t}$ \Comment{Normalize by sequence length}
  \State \textbf{Step 3: Exponentiate to get sequence-level IS weight}
  \State \quad $w_i(\theta) = \exp\left( \log w_i(\theta) \right)$
  \State \textbf{Step 4: Compute unclipped token-level loss term}
  \State \quad $L_{i,t}(\theta) = w_i(\theta) \cdot A_{i,t} \cdot \text{Mask}_{i,t}$ \Comment{Weight advantage by sequence-level IS ratio}
\EndFor

\State \textbf{Step 5: Calculate normalized policy objective}
\State \quad $J(\theta) = \frac{1}{M} \sum_{i=1}^G \sum_{t=1}^{|y_i|} L_{i,t}(\theta)$ \Comment{Normalize by total valid tokens to avoid batch bias}

\State \textbf{Step 6: Compute gradient of the policy objective}
\State \quad $\nabla_\theta J(\theta) = \frac{1}{M} \sum_{i=1}^G \sum_{t=1}^{|y_i|} \nabla_\theta L_{i,t}(\theta)$
\State \quad $\nabla_\theta L_{i,t}(\theta) = A_{i,t} \cdot \text{Mask}_{i,t} \cdot \frac{w_i(\theta)}{|y_i| \cdot \pi_\theta(y_{i,t} \mid x_i, y_{i,<t})} \cdot \nabla_\theta \pi_\theta(y_{i,t} \mid x_i, y_{i,<t})$ \Comment{Chain rule for gradient}

\State \textbf{Step 7: Update policy parameters (gradient ascent)}
\State \quad $\theta \leftarrow \theta + \alpha \cdot \nabla_\theta J(\theta)$

\Return Updated policy parameters $\theta$
\end{algorithmic}
\end{algorithm*}
\subsection{Step 1: Taylor Expansion for Entropy Change}
For $\Delta\theta$, the entropy change $\Delta\Hent = \Hent(\pi_{\theta+\Delta\theta}) - \Hent(\pitheta)$ is approximated by:
\begin{equation}
\Delta\Hent \approx \grad_\theta \Hent(\pitheta) \cdot \Delta\theta
\label{eq:taylor}
\end{equation}

\subsection{Step 2: Gradient of Policy Entropy}
\subsubsection{Expand Entropy Definition}
Discrete policy entropy (sum over all actions):
\begin{equation}
\Hent(\pitheta) = -\sum_{\aaction} \pitheta(\aaction|\sstate) \log\pitheta(\aaction|\sstate)
\label{eq:entropy_discrete}
\end{equation}
Gradient of Eq. (\ref{eq:entropy_discrete}) as:
\begin{equation}
\begin{split}
\grad_\theta \Hent(\pitheta) &= -\sum_a \left[ \grad_\theta\pi_\theta(a|s) \log\pi_\theta(a|s) \right. \\
&\quad + \pi_\theta(a|s) \grad_\theta\log\pi_\theta(a|s) \left. \right]
\end{split}
\label{eq:entropy_grad_raw}
\end{equation}
\subsubsection{Simplify with Probability Normalization}
Since $\sum_{\aaction} \pitheta(\aaction|\sstate) = 1$, we have $\sum_{\aaction} \grad_\theta\pitheta(\aaction|\sstate) = 0$. Using $\grad_\theta\pitheta(\aaction|\sstate) = \pitheta(\aaction|\sstate) \grad_\theta\log\pitheta(\aaction|\sstate)$, the first term in Eq. (\ref{eq:entropy_grad_raw}) vanishes. Thus:
\begin{equation}
\grad_\theta \Hent(\pitheta) = -\sum_{\aaction} \pitheta(\aaction|\sstate) \grad_\theta\log\pitheta(\aaction|\sstate)
\label{eq:entropy_grad_simplified}
\end{equation}

Log-probability derivative for Softmax:
\begin{equation}
\small
\grad_\theta\log\pitheta(\aaction|\sstate) = \grad_\theta\phiscore_\theta(\aaction|\sstate) - \E_{\aaction'\sim\pitheta}[\grad_\theta\phiscore_\theta(\aaction'|\sstate)]
\label{eq:softmax_log_deriv}
\end{equation}
Substitute into Eq. (\ref{eq:entropy_grad_simplified}):
\begin{equation}
\small
\grad_\theta \Hent(\pitheta) = -E _{\pitheta} \left[ \grad_\theta\phiscore_\theta(\aaction|\sstate) - \E_{\aaction'\sim\pitheta}[\grad_\theta\phiscore_\theta(\aaction'|\sstate)] \right]
\label{eq:entropy_grad_final}
\end{equation}

\subsection{Step 3: Entropy Change with NPG Update}
\subsubsection{NPG Update Rule}
NPG update (stabilized by Fisher matrix):
\begin{equation}
\Delta\theta = \beta^{-1} \cdot \grad_\theta J(\pitheta),
\label{eq:npg_update}
\end{equation}
$J(\pitheta)$ = expected cumulative reward:
\begin{equation}
\grad_\theta J(\pitheta) = \E_{\sstate,\aaction\sim\pitheta}[\grad_\theta\log\pitheta(\aaction|\sstate) \cdot \Aadv(\sstate,\aaction)]
\label{eq:policy_gradient_appendix}
\end{equation}

\subsubsection{Final Entropy Change}
Substitute Eqs. (\ref{eq:entropy_grad_simplified}) and (\ref{eq:npg_update}) into Eq. (\ref{eq:taylor}):
\begin{equation}
\Delta\Hent \approx -\beta^{-1} \cdot \Cov_{\pitheta}\left[ \log\pitheta(\aaction|\sstate), \Aadv(\sstate,\aaction) \right]
\label{eq:entropy_change_final}
\end{equation}
where $\Cov[X,Y] = \E[XY] - \E[X]\E[Y]$.
Natural Policy Gradient (NPG) \cite{kakade2001natural} define a formula measures \textbf{how the current decision will lead to changes in entropy(See Details in the Appendix \ref{sec:Derivation of Policy Entropy Change})}:
\begin{equation} \label{eq:entropy_covariance}
\small
\begin{split}
&\mathbb{E}_{s \sim d^k} \mathcal{H}(\pi_{k+1}(\cdot \mid s)) - \mathbb{E}_{s \sim d^k} \mathcal{H}(\pi_k(\cdot \mid s)) \\
&\approx -\frac{1}{\beta} \cdot \operatorname{Cov}_{a \sim \pi_\theta^k(\cdot \mid s)} \left( \log \pi_k(a \mid s), r(s, a) \right),
\end{split}
\end{equation}
where the covariance term $\operatorname{Cov}$ tracks the entropy change. \textbf{This relationship highlights the core principle of E-E trade-off: balancing the two terms in Eq. \ref{eq:entropy_covariance}.}

\subsection{Key Conclusions}
\begin{enumerate}[leftmargin=*]
    \item $\Cov[\log\pitheta, \Aadv] > 0$: $\Delta\Hent < 0$ (entropy ↓, policy more deterministic).
    \item $\Cov[\log\pitheta, \Aadv] < 0$: $\Delta\Hent > 0$ (entropy ↑, policy more exploratory).
    \item $\Cov[\log\pitheta, \Aadv] = 0$: $\Delta\Hent = 0$ (entropy unchanged).
\end{enumerate}

This derivation underpins entropy regularization for balanced E-E trade-off in RL.

\section{Computation Graph for the Token-Level} 
\label{sec:token_level_computation_graph}

We design a carefully structured backward pass to ensure training stability and theoretical consistency, primarily by handling Importance Sampling (IS) ratios at both the sequence-level and token-level.

The computation graph above stabilizes token-level policy updates by aggregating sequence-level Importance Sampling weights, which addresses the uneven impact of entropy and KL-Divergence regularization in CoT reasoning. This design aligns with the established performance-entropy relationship $R = -a \exp(\mathcal{H}) + b$ \cite{cui2025entropy}: optimizing downstream mathematical reasoning tasks tends to reduce policy entropy (increasing determinism), while artificially pushing entropy higher rarely improves performance, and adding a global KL-Divergence term often degrades stability. The core reason is that in CoT reasoning, token distributions shift dynamically across reasoning steps, so uniform entropy/KL regularization affects tokens unevenly and can lead to training collapse.

\section{Use of LLMs}
Large language models (LLMs), specifically DeepSeek-R1 and GPT-4 Turbo (GPT-5 was not used, as it remains unreleased as of 2026), were employed solely as a writing assistance tool during the preparation of this manuscript. These LLMs were used only to refine the clarity, readability, and presentation of the text—they were not used for any research-critical tasks, including but not limited to: conceiving the research design, developing algorithms, conducting experiments, analyzing data, or interpreting results. The authors bear sole responsibility for the entire research conception, technical direction, scientific content, and interpretation of all experimental results. No LLMs were used to generate or modify experimental data, and all conclusions presented in this work are the authors' independent scientific judgments.

\end{document}